\documentclass{article}
\usepackage{spconf}
\usepackage[T1]{fontenc}
\usepackage{amsmath,amssymb,amsthm,mathtools,graphicx,booktabs}
\allowdisplaybreaks[4]
\usepackage[hidelinks,hypertexnames=false]{hyperref}
\makeatletter
\let\originalThmhead\thmhead@plain
\def\thmhead@plain#1#2#3{%
  \thmname{#1}\thmnumber{ #2}\thmnote{ {\the\thm@notefont#3}}}
\makeatother
\newtheorem{theorem}{Theorem}
\newtheorem{lemma}{Lemma}
\newtheorem{proposition}{Proposition}

\theoremstyle{definition}
\newtheorem{assumption}{Assumption}
\newtheorem{definition}{Definition}
\theoremstyle{remark}
\newtheorem{remark}{Remark}
\newcommand{\paperTitle}{Private Decentralized Optimization with Noise Reduction and Bias Correction}
\newcommand{\E}{\mathbb E}

\usepackage{algorithm,algpseudocode}
\DeclareMathSizes{11}{11}{8}{6}
\renewcommand{\paragraph}[1]{\par\addvspace{.7\baselineskip}%
  \noindent\textbf{#1}\enspace}

\title{\MakeUppercase{\paperTitle}}

\name{
Yizhao Fan$^{1,2}$\quad
  Wenjian Luo$^{1,*}$ \quad
 Jiaojiao Zhang$^{2,*}$
\address{$^1$Harbin Institute of Technology, Shenzhen \quad $^2$Great Bay University}
\thanks{Corresponding authors: Wenjian Luo and Jiaojiao Zhang.}
}

\begin{document}
\maketitle
\begin{abstract}
Private decentralized learning is affected by sampling noise, privacy
noise, and decentralized bias under heterogeneous data.
We propose Private Recursive Decentralized Optimization (PRDO).
PRDO uses recursive estimation with same-batch gradient differences to reduce estimation errors caused by sampling and privacy noise, while its Exact Diffusion component corrects decentralized bias arising from data heterogeneity.
Our analysis establishes a nonconvex convergence bound without assuming uniformly bounded data heterogeneity across nodes.
It further gives a sufficient condition under which recursive gradient
differences yield strictly lower query sensitivity than private Exact
Diffusion, together with an example that rigorously satisfies
this condition.
Experiments show improved accuracy over
the evaluated baselines.
\end{abstract}

\vspace{-2mm}
\section{Introduction}\label{src:sec-introduction}
Decentralized learning trains a shared model through neighbor-to-neighbor
communication, avoiding a coordinating server and distributing the
communication load across the network. Although each node keeps its raw data
local, exchanged updates may reveal membership information or enable reconstruction of training records~\cite{zhu2019dlg}.
Differential privacy (DP) provides rigorous record-level protection by
calibrating random perturbations to query sensitivity~\cite{dwork2014algorithmic}.
However, sampling noise, privacy noise, and decentralized bias can still
degrade learning over heterogeneous data.

Existing methods address decentralized bias and sampling error through
different mechanisms. Gradient tracking estimates the network gradient by
exchanging an additional model-sized gradient-estimation
vector~\cite{nedic2017diging}. Exact Diffusion (ED) corrects decentralized
bias arising from heterogeneous local data through adaptation, correction,
and mixing while communicating only one model-sized vector per neighbor per
round~\cite{yuan2019exact1,yuan2020bias}. Nonconvex analyses establish
network-dependent convergence guarantees for ED~\cite{alghunaim2022unified},
and recent work improves its communication efficiency~\cite{faia2025communication}.
STORM instead targets sampling error through recursive estimation based on
same-batch gradient differences~\cite{cutkosky2019storm}. GT-HSGD and
DEEPSTORM extend this variance-reduction mechanism to decentralized
optimization but do not provide record-level
DP~\cite{xin2021gthsgd,mancinoball2023deepstorm}.

Recursive estimation has also been studied under DP. DP-SRM develops
centralized private recursive momentum~\cite{wang2023dpsrm}, DIFF2 uses
server-based gradient differences~\cite{murata2023diff2}, and PriSMA develops
clipped recursive momentum in a server--client
architecture~\cite{huang2025prisma}. These methods do not address
decentralized bias correction in neighbor-only networks. Complementary
approaches exploit restricted participant views and gossip for privacy
amplification~\cite{cyffers2022decentralization,cyffers2022muffliato}, use
correlated perturbations~\cite{allouah2024decor}, or improve communication
efficiency through deep quantization~\cite{francis2025deepquantizers}.
PrivSGP-VR provides private decentralized variance reduction through
stochastic gradient push, but communicates an additional push-sum scalar
alongside the model vector and assumes uniformly bounded data
heterogeneity~\cite{zhu2024privsgpvr}.

Despite these advances, reducing sampling and privacy errors while
correcting decentralized bias over heterogeneous data remains challenging.
We develop PRDO to address this challenge.
Our contributions are threefold.

\noindent{\bf (i)} We develop PRDO with recursive same-batch gradient differences to
reduce errors induced by stochastic gradient descent  and DP noise
variance, and Exact Diffusion correction to counteract
heterogeneity-induced decentralized bias.

\noindent
{\bf (ii)} We derive a nonconvex stationarity bound
separating sampling noise, privacy noise, and initial gradient heterogeneity,
without assuming uniformly bounded data heterogeneity across nodes.
We identify a sufficient condition under which recursive gradient
differences strictly reduce query sensitivity relative to private ED
and construct an instance satisfying it.

\noindent
{\bf (iii)} Experiments show improved accuracy
over the evaluated decentralized baselines.

\vspace{-2mm}
\section{Problem and Preliminaries}\label{src:secproblem}
\subsection{Problem formulation}\label{src:sec-objective}
We consider $n$ nodes that collaboratively train a model with $d$
parameters without a coordinating server.
Node $i$ holds a local data set $D_i=\{\zeta_{ir}\}_{r=1}^N$.
The nodes jointly solve
\[
\min_{x\in\mathbb R^d} f(x):=\frac1n\sum_i f_i(x),\quad
f_i(x):=\frac1N\sum_r\ell_i(x;\zeta_{ir}).
\]
Here $\ell_i(x;\zeta_{ir})$ is the loss of model $x$ on record
$\zeta_{ir}$. The local data sets $D_i$ are heterogeneous across nodes.

Each node maintains a local model $x_i\in\mathbb R^d$, yielding the
equivalent consensus-constrained formulation
\begin{equation}\label{eq:consensus-objective}
\min_{x_1,\ldots,x_n}\ \frac1n\sum_i f_i(x_i)
\quad\text{subject to}\quad x_1=\cdots=x_n.
\end{equation}
Thus decentralized learning requires both optimization of the global
objective and agreement among local models.

The communication graph is undirected, with node set
$\mathcal V=\{1,\ldots,n\}$ and edge set $\mathcal E$.
Let $\mathcal N_i$ contain node $i$ and its neighbors.
A mixing matrix $W$ assigns weight $W_{ij}$ to node $j$'s message,
with $W_{ij}=0$ for $j\notin\mathcal N_i$.
Each node mixes its neighbors' information using the weights $W_{ij}$.
This weighted averaging promotes agreement among local models and
helps satisfy the consensus constraint.
The conditions on $W$ are stated in Assumption~\ref{ass:network}.

\subsection{Differential privacy preliminaries}\label{src:sec-dp-preliminaries}
To protect training records, each node adds DP noise to its local query
before communicating.
We protect records against an external passive eavesdropper that
observes the complete collection of messages sent over all edges and rounds, and an internal
honest-but-curious node that follows the protocol but uses received
messages to infer records held by other nodes.

Write $D=(D_1,\ldots,D_n)$.
Add/remove adjacency $D\sim D'$ means that one record is added or
removed at one node, with all other local data fixed.

\begin{definition}[Record-level differential privacy]\label{def:record-dp}
	A randomized mechanism $\mathcal A$ that releases the complete communication
	transcript is $(\varepsilon,\delta)$-DP if, for every adjacent pair
	$D\sim D'$ and every measurable set $O$ of communication transcripts,
	\begin{equation}\label{eq:dp-definition}
		\Pr[\mathcal A(D)\in O]
		\leq e^\varepsilon\Pr[\mathcal A(D')\in O]+\delta.
	\end{equation}
\end{definition}
Definition~\ref{def:record-dp} applies DP~\cite{dwork2014algorithmic} to the transcript visible to the external
or internal observer described above. Our privacy analysis uses amplification
by Poisson sampling with unreleased record
identities~\cite{mironov2019sgm} and adaptive
composition~\cite{mironov2017rdp}.

\begin{algorithm}[th]
	\caption{The proposed PRDO}
	\label{alg:storm-ed}
	\begin{algorithmic}[1]
		\Require $\alpha,\gamma,T,W,b_0,b,C_g,C_\Delta,\nu_0,\nu_{\rm dp}$;
		$x_i^0=\psi_i^0=x^0$, $\forall i$.
		\For{$t=0,\ldots,T-1$; all nodes $i$ in parallel}
		\If{$t=0$}
		\State Draw $\mathcal B_i^0$ with rate $b_0/N$.
		\State $g_i^0\gets b_0^{-1}\!\sum_{r\in\mathcal B_i^0}
		\operatorname{clip}_{C_g}(\nabla\ell_i(x^0;\zeta_{ir}))$.
		\State Draw $\xi_i^0\sim\mathcal N(0,\nu_0^2I_d)$;
		$v_i^0\gets g_i^0+\xi_i^0$.
		\Else
		\State Draw $\mathcal B_i^t$ with rate $b/N$.
		\State Compute $g_i^t$ and $d_i^t$ from $\mathcal B_i^t$.
		\State Draw $\xi_i^t\sim\mathcal N(0,\nu_{\rm dp}^2I_d)$.
		\State Update $v_i^t$ by~\eqref{eq:main-recursion}.
		\EndIf
		\State $\psi_i^{t+1}\gets x_i^t-\alpha v_i^t$ (adaptation).
		\State $\phi_i^{t+1}\gets\psi_i^{t+1}+x_i^t-\psi_i^t$ (correction).
		\State Send $\phi_i^{t+1}$ to neighbors; receive $\phi_j^{t+1}$.
		\State $x_i^{t+1}\gets\sum_{j\in\mathcal N_i}W_{ij}\phi_j^{t+1}$ (mixing).
		\EndFor
	\end{algorithmic}
\end{algorithm}

\vspace{-2mm}
\section{Proposed Algorithm}\label{src:secalgorithm}
Our PRDO uses recursive gradient differences to reduce stochastic
estimation error and Exact Diffusion correction to remove the bias
caused by heterogeneous local data sets $D_i$ in decentralized updates.
Controlling these two errors improves learning performance.
The method uses stepsize $\alpha$ and momentum parameter $\gamma$,
following stochastic recursive momentum~\cite{cutkosky2019storm}.

At round $t$, node $i$ draws an independent Poisson batch
$\mathcal B_i^t$. Each record $\zeta_{ir}\in D_i$ is selected with
probability $b/N$ after initialization, where $b$ is the expected batch size.
The gradient estimator is normalized by the batch size:
\begin{equation}\label{eq:poisson}
\widehat g_i(x;\mathcal B_i^t)=\frac1b
\sum_{r\in\mathcal B_i^t}\nabla\ell_i(x;\zeta_{ir}).
\end{equation}
For a vector $u$, define
$\operatorname{clip}_{C}(u):=u/\max\{1,\|u\|/C\}$.
All nodes start from the same public model $x^0$.
Initialization uses expected batch size $b_0$.
Let $\nu_0^2$ and $\nu_{\rm dp}^2$ denote the initialization and
recursive DP noise variances, respectively.

To reduce sampling error, a record in $\mathcal B_i^t$ is evaluated
at the current model $x_i^t$ and previous model $x_i^{t-1}$,
giving the clipped estimators
\begin{align*}
g_i^t&=\frac1b\sum_{r\in\mathcal B_i^t}
\operatorname{clip}_{C_g}(\nabla\ell_i(x_i^t;\zeta_{ir})),\\
d_i^t&=\frac1b\sum_{r\in\mathcal B_i^t}
\operatorname{clip}_{C_\Delta}(\nabla\ell_i(x_i^t;\zeta_{ir})
-\nabla\ell_i(x_i^{t-1};\zeta_{ir})).
\end{align*}
For $t\geq1$, the recursive direction is
\begin{equation}\label{eq:main-recursion}
v_i^t=(1-\gamma)v_i^{t-1}+\gamma g_i^t
+(1-\gamma)d_i^t+\gamma\xi_i^t.
\end{equation}
The new data-dependent term is $\gamma g_i^t+(1-\gamma)d_i^t$;
we call it the fresh query because it uses the newly sampled batch,
rather than the previous private direction $v_i^{t-1}$.
Same-batch evaluations share sampling fluctuations that cancel in their
difference. Under smoothness, the variance of this difference
has an upper bound proportional to the squared model change.
Together with the $\gamma$ weight on the fresh
gradient, this reduces injected sampling error when model changes are small.

To control decentralized bias, each node first adapts its model,
then corrects the update and mixes information from its neighbors.
Algorithm~\ref{alg:storm-ed} gives the complete  procedure.

For our PRDO, each node sends one $d$-vector to each neighbor per round.
Setting $\gamma=1$ recovers private Exact Diffusion (DP-ED).
For $\gamma<1$, recursive differences reduce the new sampling error
and can improve accuracy when their sensitivity is sufficiently small.
When clipping is inactive, exact gradients and zero DP noise give
$v_i^t=\nabla f_i(x_i^t)$ for every $\gamma$, so PRDO preserves
ED's deterministic trajectory.

\vspace{-2mm}
\section{Analysis}\label{src:secanalysis}
To quantify the accuracy of PRDO, we state assumptions, determine the required DP noise variance, and then bound
the resulting stationarity error.

\begin{assumption}[Smooth record losses]\label{ass:objectives}
Each record loss has an $L_s$-Lipschitz gradient:
\[
\|\nabla\ell_i(x;\zeta)-\nabla\ell_i(y;\zeta)\|
\leq L_s\|x-y\|,\quad\forall\zeta,x,y.
\]
\end{assumption}

\begin{assumption}[Communication network]\label{ass:network}
The fixed mixing matrix $W$ is symmetric, entrywise nonnegative,
doubly stochastic, positive semidefinite, and supported on an
undirected connected graph.
Its eigenvalue one is simple, and its second-largest eigenvalue
is $\lambda\in[0,1)$.
\end{assumption}

\begin{assumption}[Stochastic oracle]\label{ass:oracle}
For a point $x$ fixed before the current batch is drawn,
\begin{align}
\E_t[\widehat g_i(x;\mathcal B_i^t)]&=\nabla f_i(x),
\label{cv:oracle-unbiased}\\
\E_t\|\widehat g_i(x;\mathcal B_i^t)-\nabla f_i(x)\|^2
&\leq \sigma_{\rm sgd}^2/b.
\label{cv:oracle-variance}
\end{align}
Here $\E_t$ denotes expectation conditional on the past, and
$\sigma_{\rm sgd}^2$ is the sampling-variance scale.
At initialization, the variance bound uses $b_0$ in place of $b$.
\end{assumption}
\begin{assumption}[Inactive clipping]\label{ass:clipping}
The clipping operations in Algorithm~\ref{alg:storm-ed} are inactive
almost surely:
\begin{align}
\|\nabla\ell_i(x_i^t;\zeta)\|&\leq C_g,
\label{pv:gradient-bound}\\
\|\nabla\ell_i(x_i^t;\zeta)-\nabla\ell_i(x_i^{t-1};\zeta)\|
&\leq C_\Delta
\label{pv:difference-bound}
\end{align}
for every node, round, and record queried by the algorithm.
\end{assumption}
This regime is consistent with bounded-gradient analyses in private optimization~\cite{li2023dp2} and avoids the stochastic bias introduced by active clipping~\cite{koloskova2023clipping}. In the Appendix~\ref{app:sensitivity-example}, we construct an  example in which suitable thresholds keep both clipping operations inactive.

\subsection{Privacy analysis}\label{src:secprivacy}
To determine the DP noise, we bound the effect of one
record on the data-dependent query.
The fresh query in~\eqref{eq:main-recursion} is
\begin{equation}\label{eq:main-query}
q_i^t=\gamma g_i^t+(1-\gamma)d_i^t.
\end{equation}
For add/remove adjacent data sets $D\sim D'$, fix the same past
transcript and sampling mask $\mathcal B$. If the changed record
$\zeta$ is selected, substituting~\eqref{eq:main-query} gives
\[
\begin{aligned}
&q_i^t(D;\mathcal B)-q_i^t(D';\mathcal B)\\
&=\frac1b\Big\{\gamma\operatorname{clip}_{C_g}
(\nabla\ell_i(x_i^t;\zeta))\\
&\quad +(1-\gamma)\operatorname{clip}_{C_\Delta}
(\nabla\ell_i(x_i^t;\zeta)-\nabla\ell_i(x_i^{t-1};\zeta))\Big\}.
\end{aligned}
\]
The difference is zero if the record is not selected.
The clipping bounds and the triangle inequality yield
\[
\begin{aligned}
\Delta_{q_i^t}
&:=\sup_{D\sim D'}\|q_i^t(D;\mathcal B)-q_i^t(D';\mathcal B)\|\\
&\leq\frac{\gamma C_g+(1-\gamma)C_\Delta}{b}.
\end{aligned}
\]
We therefore define the sensitivity scale
\begin{equation}\label{eq:sensitivity}
S_\gamma:=\gamma C_g+(1-\gamma)C_\Delta,
\end{equation}
so that $\Delta_{q_i^t}\leq S_\gamma/b$ uniformly over histories and masks.
Remark~\ref{rem:sensitivity-benefit} explains when recursive estimates 
can reduce this sensitivity scale.

For $0<\varepsilon\leq1$, $0<\delta<1/2$, and the parameter schedule
used below, a sufficient choice of the DP noise variances has order
\begin{equation}\label{eq:dp-noise-variance}
\begin{aligned}
\nu_0^2&=\mathcal O\!\left(
\frac{C_g^2b^2}{b_0^2\varepsilon^2}
\left[\frac{T\log(1/\delta)}{N^2}
+\frac{\log^2(1/\delta)}{b^2}\right]\right),\\
\nu_{\rm dp}^2&=\mathcal O\!\left(
\frac{S_\gamma^2}{\gamma^2\varepsilon^2}
\left[\frac{T\log(1/\delta)}{N^2}
+\frac{\log^2(1/\delta)}{b^2}\right]\right).
\end{aligned}
\end{equation}
With~\eqref{eq:dp-noise-variance}, Algorithm~\ref{alg:storm-ed} is
record-level $(\varepsilon,\delta)$-DP for the external and internal
observers described in Section~\ref{src:sec-dp-preliminaries}.

\begin{figure*}[ht]
	\centering
	\includegraphics[width=.95\textwidth]{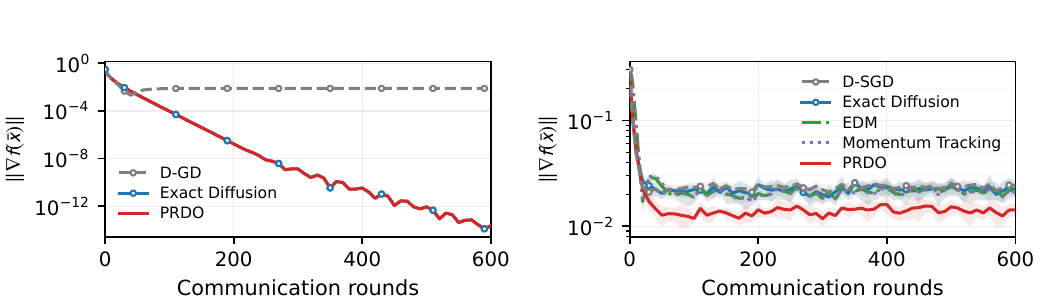}
	\vspace{-3mm}
	\caption{Nonprivate logistic regression on a 32-node ring. Left: full-gradient setting. Right: stochastic-gradient setting.}
	\label{fig:rq1}
\end{figure*}

\subsection{Convergence and privacy}\label{src:secconvergence}
With the DP noise variance determined, we now analyze convergence.
Let $\bar x^t=n^{-1}\sum_i x_i^t$ denote the network-average model.
Let $\Delta_f^0=f(x^0)-\inf_x f(x)<\infty$ and define the initial
gradient heterogeneity
\[
\zeta_0^2=\frac1n\sum_i
\|\nabla f_i(x^0)-\nabla f(x^0)\|^2.
\]

\begin{theorem}[Private convergence]\label{thm:utility}
Under Assumptions~\ref{ass:objectives}--\ref{ass:clipping},
for any $c_\gamma\in(0,1]$, there exists $c_\alpha>0$ such that,
with the DP noise variance in~\eqref{eq:dp-noise-variance} and
\[
\alpha=\frac{c_\alpha(1-\sqrt\lambda)}{L_s(T+1)^{1/3}},
\; \gamma=\frac{c_\gamma}{(T+1)^{2/3}},
\; b_0=\lceil b(T+1)^{1/3}\rceil,
\]
for $0<\varepsilon\leq1$, $0<\delta<1/2$ and horizons with
$b_0\leq N$, Algorithm~\ref{alg:storm-ed} is $(\varepsilon,\delta)$-DP and
\begin{align}
&\frac1T\sum_{t=0}^{T-1}\E\|\nabla f(\bar x^t)\|^2
=\mathcal O\Bigg(
\underbrace{\frac{L_s\Delta_f^0}{(1-\sqrt\lambda)T^{2/3}}}
_{\text{optimization}}\notag\\
&\quad+\underbrace{\frac{\lambda\zeta_0^2}{bnT(1+\sqrt\lambda)}}
_{\substack{\text{initial gradient}\\\text{heterogeneity}}}
+\underbrace{\frac{\sigma_{\rm sgd}^2}{nbT^{2/3}}
 \left(1+\frac{\lambda}{b}\right)}
_{\substack{\text{sampling and}\\\text{network propagation}}}
\label{eq:private-utility}\\
&\quad+\underbrace{
\frac{dS_\gamma^2}{n\gamma^2\varepsilon^2T^{1/3}}
 \left(1+\frac{\lambda}{b}\right)
 \left[\frac{T\log(1/\delta)}{N^2}
 +\frac{\log^2(1/\delta)}{b^2}\right]
}
_{\text{privacy and network propagation}}
\Bigg).\notag
\end{align}
\end{theorem}

\begin{remark}[Sensitivity benefit of recursion]
\label{rem:sensitivity-benefit}
The recursion in~\eqref{eq:main-recursion} reuses $v_i^{t-1}$, while
the fresh query in~\eqref{eq:main-query} replaces part of the fresh
gradient with a same-record difference. Compared with $\gamma=1$, if
$C_\Delta<C_g$ and $\gamma<1$, then
\begin{equation}\label{eq-sens}
	S_\gamma=\gamma C_g+(1-\gamma)C_\Delta<C_g.
\end{equation}
Moreover, $L_s$-smoothness and the adaptation identity give
\[
\begin{aligned}
	\|\nabla\ell_i(x_i^t;\zeta)-\nabla\ell_i(x_i^{t-1};\zeta)\|
	&\leq L_s\|x_i^t-x_i^{t-1}\|,\\
	x_i^t-x_i^{t-1}
	&=-\alpha v_i^{t-1}+(x_i^t-\psi_i^t).
\end{aligned}
\]
For general smooth objectives, these relations show when the condition \eqref{eq-sens} holds. When $x_i^t\approx\psi_i^t$, bounded directions
$\|v_i^{t-1}\|$ and a small $\alpha$ keep successive iterates
close and thus make the same-record gradient difference small. A small
$\gamma$ then further reduces $S_\gamma$.
Appendix~\ref{app:sensitivity-example} constructs a specific example for which the condition \eqref{eq-sens} holds for every
sampled batch and Gaussian history, yielding
$S_\gamma=\mathcal O(\gamma+\alpha)$.
For fixed $\gamma$ and privacy budget, a smaller sensitivity scale
$S_\gamma$ reduces the privacy-induced
error term in~\eqref{eq:private-utility}.
\end{remark}

\vspace{-2mm}
\section{Experiments}\label{src:secexperiments}
\noindent\textbf{Nonprivate logistic regression.}
We used 32 ring-connected nodes with 2000 records each and a
nonconvex regularizer. D-GD/D-SGD, Exact Diffusion, EDM~\cite{hu2025edm}, Momentum
Tracking~\cite{takezawa2023momentum}, and PRDO shared stepsize
$\alpha=1$. PRDO used $\gamma=0.1$; the momentum baselines used
history weight $0.9$. 

As in Fig.~\ref{fig:rq1}, with full gradients, Exact Diffusion and PRDO followed the same 
trajectory, as established in Appendix~\ref{app:average}, while D-GD
retained decentralized bias. 
In the minibatch comparison, PRDO attained lower late-stage
gradient norms over five paired sampling streams, indicating higher
optimization accuracy. The small $\gamma$ downweights fresh sampling
noise, while same-batch gradient differences track gradient changes,
helping reduce the error caused by SGD variance.
Momentum Tracking sent two model-sized vectors per neighbor, whereas
the other methods sent one.

\noindent\textbf{Private CIFAR-10 implementation.}
{
We compared PRDO with DP-ED, DP-EDM, DP-MT, and DP-DSGD on ten
ring-connected nodes under $\varepsilon\in\{4,6,8\}$ and
$\delta=10^{-5}$. All methods used the same VGG-style network and
heterogeneous data partition. Implementation
and privacy-accounting details are given in
Appendix~\ref{app:experiments}.

As shown in Table~\ref{tab:cifar-private}, PRDO attained higher test
accuracy than the evaluated baselines at all three privacy budgets.
The results are consistent with recursive gradient differences reducing
noise-induced estimation error while Exact Diffusion corrects model
disagreement caused by heterogeneous data.
}

\vspace{-3mm}
\begin{table}[h]
\centering
\caption{Test accuracy (\%) on CIFAR-10 under different privacy budgets.}
\label{tab:cifar-private}
{
\setlength{\tabcolsep}{3.5pt}
\begin{tabular}{lccc}
\toprule
Method & $\varepsilon=4$ & $\varepsilon=6$ & $\varepsilon=8$\\
\midrule
PRDO & $\mathbf{53.368}$ & $\mathbf{54.284}$ & $\mathbf{54.424}$\\
DP-ED & $52.964$ & $54.084$ & $54.186$\\
DP-EDM & $53.070$ & $53.948$ & $54.224$\\
DP-MT & $53.020$ & $53.966$ & $54.220$\\
DP-DSGD & $47.834$ & $48.690$ & $49.002$\\
\bottomrule
\end{tabular}
}
\end{table}
\vspace{-3mm}

\paragraph{Conclusions.}PRDO reduces sampling and privacy errors while
counteracting decentralized bias. Our analysis characterizes its conditional
DP benefit, and experiments show improved accuracy over the evaluated baselines.
Future work will remove the inactive-clipping assumption and analyze
the bias introduced when clipping is active.
\label{end:technical}
\vspace{-1mm}
\paragraph{Acknowledgment.}This work was supported by the National Natural Science Foundation of China
under Grants 62576121 (W. Luo) and 12601589 (J. Zhang).

\clearpage
\bibliographystyle{IEEEbib_shortauthors}
\bibliography{references}
\clearpage
\setlength{\oddsidemargin}{0pt}
\setlength{\evensidemargin}{0pt}
\setlength{\topmargin}{0pt}
\setlength{\headheight}{0pt}
\setlength{\headsep}{0pt}
\setlength{\textwidth}{6.5in}
\setlength{\textheight}{9in}
\onecolumn
\pagestyle{plain}
\setlength{\footskip}{24pt}
\setlength{\emergencystretch}{2em}
\makeatletter
\renewcommand{\normalsize}{\@setfontsize\normalsize{11pt}{13.2pt}%
  \abovedisplayskip 10pt plus 2pt minus 5pt
  \belowdisplayskip \abovedisplayskip
  \abovedisplayshortskip 0pt plus 3pt
  \belowdisplayshortskip 6pt plus 3pt minus 3pt}
\makeatother
\normalsize
\appendix
\makeatletter
\let\thmhead@plain\originalThmhead
\makeatother
\makeatletter
\def\@sect#1#2#3#4#5#6[#7]#8{%
  \refstepcounter{#1}%
  \edef\@svsec{\csname the#1\endcsname.\hskip 0.6em}%
  \begingroup
    \ifnum #2=1\bf\centering
      {\interlinepenalty\@M \@svsec #8\par}%
    \else\ifnum #2=2\bf
      \noindent{\interlinepenalty\@M \@svsec #8\par}%
    \else\it
      \noindent{\interlinepenalty\@M \@svsec #8\par}%
    \fi\fi
  \endgroup
  \csname #1mark\endcsname{#7}%
  \addcontentsline{toc}{#1}{\protect\numberline{\csname the#1\endcsname}#7}%
  \@tempskipa #5\relax
  \@xsect{\@tempskipa}}
\makeatother
\numberwithin{equation}{section}
\makeatletter
\@addtoreset{theorem}{section}
\@addtoreset{lemma}{section}
\@addtoreset{proposition}{section}
\@addtoreset{corollary}{section}
\@addtoreset{assumption}{section}
\@addtoreset{remark}{section}
\makeatother
\renewcommand{\thetheorem}{\thesection.\arabic{theorem}}
\renewcommand{\thelemma}{\thesection.\arabic{lemma}}
\renewcommand{\theproposition}{\thesection.\arabic{proposition}}
\renewcommand{\thecorollary}{\thesection.\arabic{corollary}}
\renewcommand{\theassumption}{\thesection.\arabic{assumption}}
\renewcommand{\theremark}{\thesection.\arabic{remark}}
\section*{Proof appendices}

Under Assumption~\ref{ass:clipping}, the clipping operations are
identities along the algorithm trajectory.  Throughout the convergence
proofs, we therefore write
\[
g_i^t=\widehat g_i(x_i^t;\mathcal B_i^t),\qquad
\widetilde g_i^t=\widehat g_i(x_i^{t-1};\mathcal B_i^t),\qquad
d_i^t=g_i^t-\widetilde g_i^t.
\]
Thus~\eqref{eq:main-recursion} takes the equivalent form
\[
v_i^t=g_i^t+(1-\gamma)(v_i^{t-1}-\widetilde g_i^t)
+\gamma\xi_i^t.
\]

\section{Notation}\label{app:setup}\label{app:conv}
\subsection{Matrix notation}
Let $X^t,V^t\in\mathbb R^{n\times d}$ have rows
$(x_i^t)^\top,(v_i^t)^\top$, respectively. Write
\[
P=I_n-\frac{\mathbf1\mathbf1^\top}{n},\quad
\bar v^t=\frac1n(V^t)^\top\mathbf1,\quad
\overline X^t=\mathbf1(\bar x^t)^\top.
\]
Here $\mathbf1$ is the all-ones vector and $P$ is the orthogonal
disagreement projector. Vector norms are Euclidean; $\|\cdot\|_F$
and $\|\cdot\|_2$ denote the Frobenius and matrix operator norms.
Let $\mathcal G(X)$ have row $\nabla f_i(x_i)^\top$ when $X$
has row $x_i^\top$. Define the estimation errors
\[
e_i^t=v_i^t-\nabla f_i(x_i^t),\quad
E^t=V^t-\mathcal G(X^t),\quad
\bar e^t=\frac1n\sum_i e_i^t,
\]
and their moments
\begin{equation}\label{cv:moment-definitions}
C_t=\E\|PX^t\|_F^2,\quad H_t=\E\|E^t\|_F^2,\quad
h_t=\E\|\bar e^t\|^2,\quad U_t=\E\|\bar v^t\|^2.
\end{equation}
Thus $C_t/n$ is the expected consensus error, and
$(nT)^{-1}\sum_tC_t$ is its average over the $T$ updates.

\subsection{Conditional moments}
Let $\mathcal F_t$ contain the initial quantities and all batches
and Gaussian draws before round $t$, and write
$\E_t[\cdot]=\E[\cdot\mid\mathcal F_t]$.
Fresh batches are independent across nodes and rounds, conditional
on the past, and are independent of all Gaussian draws.
For the combined sampling and DP variance, write
\begin{equation}\label{cv:unified-variance}
\Sigma^2=\frac{\sigma_{\rm sgd}^2}{b}+d\nu_{\rm dp}^2.
\end{equation}
The calibration~\eqref{pv:main-noise-orders}, $b_0\geq b$, and
$S_\gamma\geq\gamma C_g$ give $\nu_0\leq\nu_{\rm dp}$; hence
$\nu_{\rm dp}^2$ also upper bounds the initialization noise variance
in the convergence estimates below.

\section{Descent preliminaries}\label{app:average}
\subsection{Gradient differences}
Under Assumption~\ref{ass:objectives}, each local gradient
$\nabla f_i$ is also $L_s$-Lipschitz, since
\[
\|\nabla f_i(x)-\nabla f_i(y)\|
\leq\frac1N\sum_r
\|\nabla\ell_i(x;\zeta_{ir})-\nabla\ell_i(y;\zeta_{ir})\|
\leq L_s\|x-y\|.
\]
Averaging over nodes gives the same bound for $\nabla f$.
For two points fixed before drawing the same batch, independence
of the Poisson inclusion events cancels the cross terms and gives
\begin{align}
&\E_t\|\widehat g_i(x;\mathcal B_i^t)-\widehat g_i(y;\mathcal B_i^t)
-\nabla f_i(x)+\nabla f_i(y)\|^2\notag\\
&=\frac{1-b/N}{bN}\sum_r
\|\nabla\ell_i(x;\zeta_{ir})-\nabla\ell_i(y;\zeta_{ir})\|^2
\leq\frac{L_s^2}{b}\|x-y\|^2.
\label{cv:oracle-difference}
\end{align}
We use~\eqref{cv:oracle-difference} to bound the gradient-difference
term in the error recursion~\eqref{cv:error-decomposition}.

\subsection{Descent bound}
Eliminating the auxiliary variables in Algorithm~\ref{alg:storm-ed} gives
\begin{equation}\label{cv:second-order}
X^{t+1}=W(2X^t-X^{t-1}-\alpha V^t+\alpha V^{t-1}),
\end{equation}
where $X^{-1}=X^0$ and $V^{-1}=0$ encode the initialization.

\begin{lemma}[Average dynamics and inexact-direction descent]
\label{cv:lemma-average}
Under Assumptions~\ref{ass:objectives}--\ref{ass:clipping},
the initialized algorithm satisfies
\begin{equation}\label{cv:average-dynamics}
\bar x^{t+1}=\bar x^t-\alpha\bar v^t.
\end{equation}
If $\alpha\leq1/(2L_s)$, then
\begin{equation}\label{cv:descent}
\mathbb E f(\bar x^{t+1})
\leq \mathbb E f(\bar x^t)
-\frac\alpha2\mathbb E\|\nabla f(\bar x^t)\|^2
-\frac\alpha4U_t+\alpha h_t+\frac{\alpha L_s^2}{n}C_t.
\end{equation}
\end{lemma}

\begin{proof}
\emph{Step 1: average dynamics.}
Taking node averages in~\eqref{cv:second-order} and using
$\mathbf1^\top W=\mathbf1^\top$ gives
\[
\bar x^{t+1}-\bar x^t
=\bar x^t-\bar x^{t-1}-\alpha\bar v^t+\alpha\bar v^{t-1}.
\]
At $t=0$, initialization gives $\bar x^{-1}=\bar x^0$ and
$\bar v^{-1}=0$, so
\[
\bar x^1-\bar x^0
=\bar x^0-\bar x^{-1}-\alpha\bar v^0+\alpha\bar v^{-1}
=-\alpha\bar v^0.
\]
Suppose
$\bar x^t-\bar x^{t-1}=-\alpha\bar v^{t-1}$.
Substituting this induction hypothesis into the recurrence gives
\[
\bar x^{t+1}-\bar x^t
=-\alpha\bar v^{t-1}-\alpha\bar v^t+\alpha\bar v^{t-1}
=-\alpha\bar v^t.
\]
This proves~\eqref{cv:average-dynamics} by induction.

\emph{Step 2: separate estimation and disagreement errors.}
Since $f$ is $L_s$-smooth, its one-step upper bound and
\eqref{cv:average-dynamics} imply
\begin{align*}
f(\bar x^{t+1})
&\leq f(\bar x^t)-\alpha\langle\nabla f(\bar x^t),\bar v^t\rangle
+\frac{L_s\alpha^2}{2}\|\bar v^t\|^2\\
&=f(\bar x^t)-\frac\alpha2\|\nabla f(\bar x^t)\|^2
-\frac{\alpha(1-L_s\alpha)}2\|\bar v^t\|^2
+\frac\alpha2\|\bar v^t-\nabla f(\bar x^t)\|^2.
\end{align*}
The equality is the polarization identity
$-2\langle a,b\rangle=-\|a\|^2-\|b\|^2+\|a-b\|^2$.
To bound the last term, separate the estimation error from the
difference between local and common-model gradients.
Use $\bar v^t=n^{-1}\sum_i\nabla f_i(x_i^t)+\bar e^t$ and
$\nabla f(\bar x^t)=n^{-1}\sum_i\nabla f_i(\bar x^t)$.
\begin{align*}
\|\bar v^t-\nabla f(\bar x^t)\|^2
&\leq2\|\bar e^t\|^2
+2\left\|\frac1n\sum_i[\nabla f_i(x_i^t)-\nabla f_i(\bar x^t)]\right\|^2\\
&\leq2\|\bar e^t\|^2
+\frac2n\sum_i\|\nabla f_i(x_i^t)-\nabla f_i(\bar x^t)\|^2\\
&\leq2\|\bar e^t\|^2+\frac{2L_s^2}{n}\|PX^t\|_F^2.
\end{align*}
The three inequalities follow from the squared triangle inequality,
Jensen's inequality, and $L_s$-Lipschitzness, respectively.
Insert this estimate into the smoothness bound and use
$1-L_s\alpha\geq1/2$.
Taking expectations then gives~\eqref{cv:descent}, including the
negative direction term needed for the later absorption step.
\end{proof}

\paragraph{Preservation of the exact-gradient trajectory.}
The average identity alone does not establish equivalence of the
local trajectories.
For that purpose, consider full gradients and zero DP noise, with
$v_i^0=\nabla f_i(x^0)$.
If $v_i^{t-1}=\nabla f_i(x_i^{t-1})$, then~\eqref{eq:main-recursion} gives
\[
v_i^t=\nabla f_i(x_i^t)
+(1-\gamma)[\nabla f_i(x_i^{t-1})-\nabla f_i(x_i^{t-1})]
=\nabla f_i(x_i^t).
\]
Induction establishes the identity at every node and every
available iteration.
The updates in Algorithm~\ref{alg:storm-ed} then give the entire ED trajectory,
even when local gradients differ at initialization.

\section{Estimation error bounds}\label{app:storm}
The descent inequality~\eqref{cv:descent} requires the averaged
estimation error $h_t$.
Disagreement instead depends on the total local estimation error
$H_t$, which does not benefit directly from node averaging.
We derive both recursions from the same decomposition into historical
error and fresh innovation.
A bound on $\sum_t\E\|X^t-X^{t-1}\|_F^2$ then expresses the
model-movement input through $C_t$ and $U_t$.

\begin{lemma}[Initialization, error recursions, and movement]
\label{cv:lemma-storm}
Under Assumptions~\ref{ass:objectives}--\ref{ass:clipping}, the initial errors satisfy
\begin{equation}\label{cv:initial-errors}
h_0\leq\frac{\sigma_{\rm sgd}^2}{nb_0}+\frac{d\nu_{\rm dp}^2}{n},\quad
H_0\leq\frac{n\sigma_{\rm sgd}^2}{b_0}+nd\nu_{\rm dp}^2.
\end{equation}
For $1\leq t\leq T-1$,
\begin{align}
h_t&\leq(1-\gamma)^2h_{t-1}
+\frac{2\gamma^2\sigma_{\rm sgd}^2}{nb}
+\frac{2(1-\gamma)^2L_s^2}{bn^2}\mathbb E\|X^t-X^{t-1}\|_F^2
+\frac{\gamma^2d\nu_{\rm dp}^2}{n},
\label{cv:h-recursion}\\
H_t&\leq(1-\gamma)^2H_{t-1}
+\frac{2\gamma^2n\sigma_{\rm sgd}^2}{b}
+\frac{2(1-\gamma)^2L_s^2}{b}\mathbb E\|X^t-X^{t-1}\|_F^2
+\gamma^2nd\nu_{\rm dp}^2.
\label{cv:H-recursion}
\end{align}
Summing the model movements gives
\begin{equation}\label{cv:movement}
\sum_{t=1}^{T-1}\mathbb E\|X^t-X^{t-1}\|_F^2
\leq4\sum_{t=0}^{T-1}C_t+n\alpha^2\sum_{t=0}^{T-1}U_t.
\end{equation}
Consequently,
\begin{align}
\sum_{t=0}^{T-1}h_t
&\leq\frac{h_0}{\gamma}
+\frac{2\gamma\sigma_{\rm sgd}^2T}{nb}
+\frac{8L_s^2}{bn^2\gamma}\sum_{t=0}^{T-1}C_t
+\frac{2L_s^2\alpha^2}{bn\gamma}\sum_{t=0}^{T-1}U_t
+\frac{\gamma d\nu_{\rm dp}^2T}{n},
\label{cv:h-sum}\\
\sum_{t=0}^{T-1}H_t
&\leq\frac{H_0}{\gamma}
+\frac{2\gamma n\sigma_{\rm sgd}^2T}{b}
+\frac{8L_s^2}{b\gamma}\sum_{t=0}^{T-1}C_t
+\frac{2nL_s^2\alpha^2}{b\gamma}\sum_{t=0}^{T-1}U_t
+\gamma nd\nu_{\rm dp}^2T.
\label{cv:H-sum}
\end{align}
\end{lemma}

\begin{proof}
\emph{Step 1: initial error moments.}
At initialization the two terms of
\[
e_i^0=\widehat g_i(x^0;\mathcal B_i^0)-\nabla f_i(x^0)+\xi_i^0
\]
are centered and independent, so
$\mathbb E\|e_i^0\|^2\leq\sigma_{\rm sgd}^2/b_0+d\nu_{\rm dp}^2$.
The errors at distinct nodes are also centered and independent. Summing
over nodes gives the $H_0$ estimate, whereas
$\mathbb E\|n^{-1}\sum_i e_i^0\|^2=n^{-2}\sum_i\mathbb E\|e_i^0\|^2$
gives the $h_0$ estimate in\,\eqref{cv:initial-errors}.

\emph{Step 2: conditional innovation variances.}
Subtracting $\nabla f_i(x_i^t)$ from\,\eqref{eq:main-recursion} gives
\begin{align}
e_i^t={}&(1-\gamma)e_i^{t-1}
+\gamma[\widehat g_i(x_i^t;\mathcal B_i^t)-\nabla f_i(x_i^t)]
+\gamma\xi_i^t\notag\\
&+(1-\gamma)[\widehat g_i(x_i^t;\mathcal B_i^t)
-\widehat g_i(x_i^{t-1};\mathcal B_i^t)
-\nabla f_i(x_i^t)+\nabla f_i(x_i^{t-1})].
\label{cv:error-decomposition}
\end{align}
Both model points and the historical error are
$\mathcal F_t$-measurable.
Conditional unbiasedness at both points makes the fresh innovation
centered:
$\mathbb E_t[e_i^t-(1-\gamma)e_i^{t-1}]=0$.
The Gaussian cross term vanishes by independence and zero mean.
The two sampling terms use the same batch and need not be independent,
so we bound their sum by the squared triangle inequality:
\begin{align}
\E_t\|e_i^t-(1-\gamma)e_i^{t-1}\|^2
\leq{}&2\gamma^2\E_t
 \|\widehat g_i(x_i^t;\mathcal B_i^t)-\nabla f_i(x_i^t)\|^2
+\gamma^2d\nu_{\rm dp}^2\notag\\
&+2(1-\gamma)^2\E_t\|\widehat g_i(x_i^t;\mathcal B_i^t)
-\widehat g_i(x_i^{t-1};\mathcal B_i^t)
-\nabla f_i(x_i^t)+\nabla f_i(x_i^{t-1})\|^2\notag\\
\leq{}&\frac{2\gamma^2\sigma_{\rm sgd}^2}{b}
+\frac{2(1-\gamma)^2L_s^2}{b}\|x_i^t-x_i^{t-1}\|^2
+\gamma^2d\nu_{\rm dp}^2.
\label{cv:innovation-bound}
\end{align}
The ordinary variance bound~\eqref{cv:oracle-variance} controls the
fresh gradient term.
The derived same-batch difference bound~\eqref{cv:oracle-difference}
controls the movement-dependent term.
Because the historical error is measurable with respect to the
past, its conditional cross term with the centered innovation vanishes:
\[
\mathbb E_t\|e_i^t\|^2
=(1-\gamma)^2\|e_i^{t-1}\|^2+\mathbb E_t\|e_i^t-(1-\gamma)e_i^{t-1}\|^2.
\]
Summing over nodes and taking total expectation proves
\eqref{cv:H-recursion}. For the mean error, conditional independence of
the fresh nodewise innovations gives instead
\[
\mathbb E_t\|\bar e^t\|^2
=(1-\gamma)^2\|\bar e^{t-1}\|^2
+\frac1{n^2}\sum_i\mathbb E_t\|e_i^t-(1-\gamma)e_i^{t-1}\|^2.
\]
Substitution of\,\eqref{cv:innovation-bound} and total expectation prove
\eqref{cv:h-recursion}.

\emph{Step 3: cumulative model movement.}
Orthogonality of the consensus and disagreement
subspaces and\,\eqref{cv:average-dynamics} yield
\begin{align*}
\|X^t-X^{t-1}\|_F^2
&=\|P(X^t-X^{t-1})\|_F^2
+n\|\bar x^t-\bar x^{t-1}\|^2\\
&\leq2\|PX^t\|_F^2+2\|PX^{t-1}\|_F^2
+n\alpha^2\|\bar v^{t-1}\|^2.
\end{align*}
Summing and taking expectations give
\[
\sum_{t=1}^{T-1}\mathbb E\|X^t-X^{t-1}\|_F^2
\leq2\sum_{t=1}^{T-1}C_t+2\sum_{t=0}^{T-2}C_t
+n\alpha^2\sum_{t=0}^{T-2}U_t,
\]
which implies\,\eqref{cv:movement} after enlarging the nonnegative sums.

\emph{Step 4: sum the error recursions, retaining endpoints.}
Summing\,\eqref{cv:h-recursion} from one to $T-1$ gives exactly
\begin{align*}
(1-(1-\gamma)^2)\sum_{t=0}^{T-1}h_t
&\leq h_0-(1-\gamma)^2 h_{T-1}+(T-1)\left(\frac{2\gamma^2\sigma_{\rm sgd}^2}{nb}+\frac{\gamma^2d\nu_{\rm dp}^2}{n}\right)\\ &\quad+\frac{2(1-\gamma)^2 L_s^2}{bn^2}\sum_{t=1}^{T-1}\mathbb E\|X^t-X^{t-1}\|_F^2.
\end{align*}
The exact contraction gap is
$1-(1-\gamma)^2=\gamma(2-\gamma)\geq\gamma$.
For the scheduled rate, this lower bound suffices, so we discard
the negative endpoint and enlarge $T-1$ to $T$.
Bounding $(1-\gamma)^2$ by one in the remaining nonnegative numerators
then gives
\[
\sum_{t=0}^{T-1}h_t
\leq\frac{h_0}{\gamma}
+\frac{2\gamma\sigma_{\rm sgd}^2T}{nb}
+\frac{\gamma d\nu_{\rm dp}^2T}{n}
+\frac{2L_s^2}{bn^2\gamma}\sum_{t=1}^{T-1}\mathbb E\|X^t-X^{t-1}\|_F^2.
\]
Inserting\,\eqref{cv:movement} proves\,\eqref{cv:h-sum}. Similarly,
\eqref{cv:H-recursion} sums to
\begin{align*}
(1-(1-\gamma)^2)\sum_{t=0}^{T-1}H_t
&\leq H_0-(1-\gamma)^2 H_{T-1}+(T-1)\left(\frac{2\gamma^2n\sigma_{\rm sgd}^2}{b}+\gamma^2nd\nu_{\rm dp}^2\right)\\ &\quad+\frac{2(1-\gamma)^2 L_s^2}{b}\sum_{t=1}^{T-1}\mathbb E\|X^t-X^{t-1}\|_F^2.
\end{align*}
Apply the same contraction and endpoint bounds, then insert
the movement estimate~\eqref{cv:movement}, to obtain~\eqref{cv:H-sum}.
\end{proof}

The direct SGD and DP terms in~\eqref{cv:h-sum} both decrease
with $\gamma$ when $\sigma_{\rm sgd}^2$ and $\nu_{\rm dp}^2$ are fixed.
The initialization and movement terms instead contain $1/\gamma$,
which expresses the cost of slow error contraction.
The network estimate below controls movement feedback before
these competing effects enter the final stationarity bound.

\section{Network error bounds}\label{app:network}
To bound $\sum_t h_t$ and $\sum_t H_t$
in~\eqref{cv:h-sum}--\eqref{cv:H-sum}, we need to bound
the accumulated disagreement $\sum_tC_t$.
A common-model reference separates local estimation error from
the movement of deterministic gradients.
ED's correction state then admits a contracting quadratic energy
with explicit dependence on the network spectrum.

Define the local-gradient stack at the common model by
\[
\mathcal M^t:=\mathcal G(\overline X^t),
\]
and the residual direction $R^t:=V^t-\mathcal M^t$.
The reference matrix $\mathcal M^t$ evaluates each local objective at the
averaged model $\bar x^t$.

\begin{lemma}[Network energy and cumulative disagreement]
\label{cv:lemma-network}
Under Assumptions~\ref{ass:objectives}--\ref{ass:clipping}, there is a nonnegative quadratic
energy $\mathcal E_t$ with
\begin{equation}\label{cv:initial-energy}
\mathcal E_0\leq\frac{n\alpha^2\lambda}{2(1-\lambda)}\zeta_0^2.
\end{equation}
It satisfies
\begin{align}
\mathcal E_{t+1}
&\leq\sqrt\lambda\,\mathcal E_t
+\frac{\alpha^2\lambda}{1-\sqrt{\lambda}}\|PR^t\|_F^2
+\frac{\alpha^2\lambda}{(1-\sqrt{\lambda})(1-\lambda)}\|P(\mathcal M^{t+1}-\mathcal M^t)\|_F^2,
\label{cv:network-energy}\\
\|PX^t\|_F^2&\leq4\mathcal E_t.
\label{cv:energy-coercivity}
\end{align}
If $\alpha\leq(1-\sqrt{\lambda})/(4L_s)$, define
\begin{equation}\label{cv:network-constants}
K_H:=\frac{16\alpha^2\lambda}{(1-\sqrt{\lambda})^2},\quad
K_U:=\frac{8n\alpha^4L_s^2\lambda}{(1-\sqrt{\lambda})^2(1-\lambda)}.
\end{equation}
Then
\begin{equation}\label{cv:network-sum}
\sum_{t=0}^{T-1}C_t\leq
\underbrace{\frac{4n\alpha^2\lambda\zeta_0^2}
{(1-\sqrt\lambda)(1-\lambda)}}_{\text{initial gradient heterogeneity}}
+K_H\sum_{t=0}^{T-1}H_t
+K_U\sum_{t=0}^{T-1}U_t.
\end{equation}
For $\lambda=0$, the same conclusions hold with zero network energy,
$K_H=K_U=0$, and $C_t=0$.
\end{lemma}

\begin{proof}
\emph{Step 1: isolate the corrected network state.}
Use the common-model gradients $\mathcal M^t$ and residuals $R^t$
defined above. Set $S_0=\alpha W\mathcal M^0$ and consider
\begin{align}
X^{t+1}&=(2W-I_n)X^t-S_t-\alpha WR^t,
\label{cv:augmented-X}\\
S_{t+1}&=S_t+(I_n-W)X^t
+\alpha W(\mathcal M^{t+1}-\mathcal M^t).
\label{cv:augmented-D}
\end{align}
At initialization, the first equation gives
$X^1=X^0-\alpha WV^0$, since $WX^0=X^0$.
Subtracting successive instances of the first equation and using
the second recovers~\eqref{cv:second-order}.
Thus this representation is exactly the initialized algorithm,
not an additional update rule.

\emph{Step 2: construct the energy one network mode at a time.}
If $\lambda=0$, mixing removes all model disagreement and the
claim follows with zero energy. Otherwise, let $u_z$ be an
orthonormal disagreement eigenvector of $W$ with positive eigenvalue $z$.
Zero eigenmodes have zero model and correction coordinates by
initialization and the recursions above, so they need not be inverted.
Write
\[
p_z^t=u_z^\top X^t,\quad
q_z^t=\frac{u_z^\top S_t}{\sqrt{1-z}},\quad
Z_z^t=\begin{bmatrix}p_z^t\\q_z^t\end{bmatrix}.
\]
Projecting the coupled recursion onto this eigenvector gives
\[
Z_z^{t+1}
=A_zZ_z^t+\alpha
\begin{bmatrix}
-z u_z^\top R^t\\
\displaystyle\frac{z}{\sqrt{1-z}}u_z^\top(\mathcal M^{t+1}-\mathcal M^t)
\end{bmatrix},
\quad
A_z=\begin{bmatrix}
2z-1&-\sqrt{1-z}\\
\sqrt{1-z}&1
\end{bmatrix}.
\]
Define
\[
H_z=\frac1{2z}\begin{bmatrix}
1&\sqrt{1-z}\\
\sqrt{1-z}&1
\end{bmatrix},\quad
\mathcal E_t=\sum_{z>0}\|Z_z^t\|_{H_z}^2,
\]
where $\|Z\|_{H_z}^2=\operatorname{tr}(Z^\top H_zZ)$
and the sum includes all positive disagreement modes.
Direct multiplication gives $A_z^\top H_zA_z=zH_z$.
The smaller eigenvalue of $H_z$ is
$[2(1+\sqrt{1-z})]^{-1}\geq1/4$.
Consequently,
\[
\|A_zZ_z^t\|_{H_z}^2=z\|Z_z^t\|_{H_z}^2,\quad
\|PX^t\|_F^2\leq4\mathcal E_t.
\]
This proves~\eqref{cv:energy-coercivity}.

\emph{Step 3: bound the forcing and initial energy.}
For one model coordinate, write $r,d$ for the corresponding entries
of $u_z^\top R^t$ and $u_z^\top(\mathcal M^{t+1}-\mathcal M^t)$.
The forcing norm satisfies
\[
\alpha^2
\left\|\begin{bmatrix}-zr\\zd/\sqrt{1-z}\end{bmatrix}\right\|_{H_z}^2
=\alpha^2\left(\frac z2r^2-zrd+\frac{z}{2(1-z)}d^2\right)
\leq\alpha^2\left(zr^2+\frac{z}{1-z}d^2\right).
\]
Young's inequality, with contraction factor $\sqrt\lambda$,
therefore yields after summing all modes
\[
\mathcal E_{t+1}
\leq\sqrt\lambda\,\mathcal E_t
+\frac{\alpha^2\lambda}{1-\sqrt\lambda}\|PR^t\|_F^2
+\frac{\alpha^2\lambda}{(1-\sqrt\lambda)(1-\lambda)}
\|P(\mathcal M^{t+1}-\mathcal M^t)\|_F^2.
\]
This is~\eqref{cv:network-energy}.
Initially $p_z^0=0$ and
$q_z^0=\alpha z u_z^\top\mathcal M^0/\sqrt{1-z}$, hence
\[
\mathcal E_0
=\frac{\alpha^2}{2}\sum_{z>0}\frac{z}{1-z}
\|u_z^\top\mathcal M^0\|^2
\leq\frac{\alpha^2\lambda}{2(1-\lambda)}
\|P\mathcal M^0\|_F^2
=\frac{n\alpha^2\lambda}{2(1-\lambda)}\zeta_0^2.
\]
This establishes~\eqref{cv:initial-energy}; the stochastic
initial error remains in $R^0$, separate from the heterogeneity energy.

\emph{Step 4: express the forcing through the error moments.}
Since $R^t=E^t+\mathcal G(X^t)-\mathcal G(\overline X^t)$,
smoothness and projection nonexpansiveness give
\begin{align}
\E\|PR^t\|_F^2&\leq2H_t+2L_s^2C_t,
\label{cv:R-force}\\
\E\|P(\mathcal M^{t+1}-\mathcal M^t)\|_F^2
&\leq nL_s^2\E\|\bar x^{t+1}-\bar x^t\|^2
=n\alpha^2L_s^2U_t.
\label{cv:N-force}
\end{align}
The last identity follows from~\eqref{cv:average-dynamics}.

\emph{Step 5: sum the energy and absorb disagreement.}
Taking expectations in the energy recursion, summing, and dropping
the negative terminal energy gives
\[
\begin{aligned}
(1-\sqrt\lambda)\sum_t\E\mathcal E_t
\leq{}&\mathcal E_0
+\frac{2\alpha^2\lambda}{1-\sqrt\lambda}\sum_tH_t
+\frac{2\alpha^2\lambda L_s^2}{1-\sqrt\lambda}\sum_tC_t\\
&+\frac{n\alpha^4L_s^2\lambda}
{(1-\sqrt\lambda)(1-\lambda)}\sum_tU_t.
\end{aligned}
\]
Here and below $\sum_t$ runs over the $T$ updates.
Using $C_t\leq4\E\mathcal E_t$, we obtain
\begin{equation}\label{cv:network-preabsorption}
\begin{aligned}
\sum_tC_t\leq{}&\frac{4\mathcal E_0}{1-\sqrt\lambda}
+\frac{8\alpha^2\lambda}{(1-\sqrt\lambda)^2}\sum_tH_t
+\frac{8\alpha^2L_s^2\lambda}{(1-\sqrt\lambda)^2}\sum_tC_t\\
&+\frac{4n\alpha^4L_s^2\lambda}
{(1-\sqrt\lambda)^2(1-\lambda)}\sum_tU_t.
\end{aligned}
\end{equation}
The stepsize condition makes the coefficient of $\sum_tC_t$
on the right at most $1/2$. Moving that term to the left and
using~\eqref{cv:initial-energy} gives~\eqref{cv:network-sum}
with the stated $K_H,K_U$.
\end{proof}

The coefficient $K_H$ quantifies how ED propagates local estimation
errors into disagreement, while $K_U$ accounts for reference-gradient
movement.
Together with the exact-gradient identity, this separates ED's
bias-correction role from STORM's stochastic estimation role.
The network energy controls this interaction without removing the
error feedback identified in the error recursions.

The bounds on $\sum_t h_t$ and $\sum_t H_t$ depend on
$\sum_t\E\|X^t-X^{t-1}\|_F^2$, which is bounded through
$C_t,U_t$, while~\eqref{cv:network-sum} depends on $\sum_tH_t$.
The next section closes these inequalities by
absorbing disagreement first and the remaining direction term second.

\section{Finite-time convergence}\label{app:closure}
The error bounds~\eqref{cv:h-sum}--\eqref{cv:H-sum} and the
disagreement bound~\eqref{cv:network-sum} depend on each other
through $\sum_t\E\|X^t-X^{t-1}\|_F^2$.
We close them in two stages, matching the two negative margins
available in the proof.
First, the network inequality absorbs its own disagreement feedback.
Second, the negative direction term in~\eqref{cv:descent} absorbs
the remaining average movement.
Neither stage requires recursive directions to be independent of
the current iterates.

For this calculation, $Q$ collects the disagreement terms entering
average descent.
The coefficients $A_C$ and $A_U$ collect feedback from disagreement
and average movement in the network estimate.
The remaining initialization and variance contributions are grouped
in $B_{\rm net}(T)$:

\begin{align}
Q&:=\frac{L_s^2}{n}+\frac{8L_s^2}{bn^2\gamma},\quad
A_C:=\frac{8K_HL_s^2}{b\gamma},\quad
A_U:=\frac{2K_HnL_s^2\alpha^2}{b\gamma}+K_U,
\label{cv:closure-coefficients}\\
B_{\rm net}(T)&:=
\frac{4n\alpha^2\lambda\zeta_0^2}{(1-\sqrt\lambda)(1-\lambda)}
+K_H\left(\frac{H_0}{\gamma}
+2\gamma n\Sigma^2T\right).
\label{cv:network-budget}
\end{align}
\begin{theorem}[Finite-time stationarity, direction, and consensus bounds]
\label{cv:finite-theorem}
Under Assumptions~\ref{ass:objectives}--\ref{ass:clipping}, suppose
\begin{equation}\label{cv:stability}
\alpha\leq\min\left\{\frac1{2L_s},\frac{1-\sqrt{\lambda}}{4L_s}\right\},\quad
A_C\leq\frac12,\quad
\frac{2L_s^2\alpha^2}{bn\gamma}+2QA_U\leq\frac18.
\end{equation}
Then
\begin{align}
\sum_{t=0}^{T-1}C_t&\leq2B_{\rm net}(T)+2A_U\sum_{t=0}^{T-1}U_t,
\label{cv:closed-consensus}\\
\frac1T\sum_{t=0}^{T-1}\mathbb E\|\nabla f(\bar x^t)\|^2
&\leq\frac{2\Delta_f^0}{\alpha T}+\frac{2h_0}{\gamma T}
+\frac{4\gamma\Sigma^2}{n}+\frac{4QB_{\rm net}(T)}{T},
\label{cv:finite-bound}\\
\frac{\sum_{t=0}^{T-1}U_t}{T}
&\leq\frac{8\Delta_f^0}{\alpha T}+\frac{8h_0}{\gamma T}
+\frac{16\gamma\Sigma^2}{n}+\frac{16QB_{\rm net}(T)}{T}.
\label{cv:finite-direction}
\end{align}
Thus a fully explicit finite-time consensus bound is
\begin{equation}\label{cv:finite-consensus}
\frac{\sum_{t=0}^{T-1}C_t}{nT}\leq\frac{2B_{\rm net}(T)}{nT}
+\frac{2A_U}{n}\left[
\frac{8\Delta_f^0}{\alpha T}+\frac{8h_0}{\gamma T}
+\frac{16\gamma\Sigma^2}{n}+\frac{16QB_{\rm net}(T)}{T}\right].
\end{equation}

\end{theorem}

\begin{proof}
\emph{Step 1: close the accumulated disagreement.}
Since $2\sigma_{\rm sgd}^2/b+d\nu_{\rm dp}^2\leq2\Sigma^2$,
\eqref{cv:H-sum} implies
\[
\sum_{t=0}^{T-1}H_t\leq\frac{H_0}{\gamma}+2\gamma n\Sigma^2T
+\frac{8L_s^2}{b\gamma}\sum_{t=0}^{T-1}C_t
+\frac{2nL_s^2\alpha^2}{b\gamma}\sum_{t=0}^{T-1}U_t.
\]
Insert this inequality into the network sum~\eqref{cv:network-sum}.
Collecting the disagreement and direction terms gives the
coefficients $A_C$ and $A_U$, respectively.
The remaining contributions, including initial gradient heterogeneity,
are precisely $B_{\rm net}(T)$.
The resulting inequality is
\[
\sum_{t=0}^{T-1}C_t\leq B_{\rm net}(T)+A_C\sum_{t=0}^{T-1}C_t+A_U\sum_{t=0}^{T-1}U_t.
\]
Since $A_C\leq1/2$, division by $1-A_C$ gives
\eqref{cv:closed-consensus}.

\emph{Step 2: substitute into average descent and absorb movement.}
Sum\,\eqref{cv:descent} over all $T$ updates and use the lower bound
$\mathbb E f(\bar x^T)\geq\inf_x f(x)$ to get
\[
\frac\alpha2\sum_{t=0}^{T-1}\mathbb E\|\nabla f(\bar x^t)\|^2
+\frac\alpha4\sum_{t=0}^{T-1}U_t
\leq\Delta_f^0+\alpha \sum_{t=0}^{T-1}h_t+\frac{\alpha L_s^2}{n}\sum_{t=0}^{T-1}C_t.
\]
Applying\,\eqref{cv:h-sum} and the same variance enlargement gives
\begin{align}
\frac\alpha2\sum_t\mathbb E\|\nabla f(\bar x^t)\|^2
+\frac\alpha4\sum_{t=0}^{T-1}U_t
&\leq\Delta_f^0+\frac{\alpha h_0}{\gamma}
+\frac{2\alpha\gamma\Sigma^2T}{n}
+\alpha Q \sum_{t=0}^{T-1}C_t
+\frac{2\alpha L_s^2\alpha^2}{bn\gamma}\sum_{t=0}^{T-1}U_t.
\label{cv:descent-preclosure}
\end{align}
All sums in this display range from zero to $T-1$.
Substitute\,\eqref{cv:closed-consensus}. The total right-hand
coefficient of $\sum_{t=0}^{T-1}U_t$ is exactly
\[
\alpha\left(\frac{2L_s^2\alpha^2}{bn\gamma}+2QA_U\right).
\]
The final condition in\,\eqref{cv:stability} therefore leaves at least
$\alpha \sum_{t=0}^{T-1}U_t/8$ on the left, proving the stronger joint inequality
\begin{equation}\label{cv:joint-bound}
\frac\alpha2\sum_t\mathbb E\|\nabla f(\bar x^t)\|^2
+\frac\alpha8\sum_{t=0}^{T-1}U_t
\leq\Delta_f^0+\frac{\alpha h_0}{\gamma}
+\frac{2\alpha\gamma\Sigma^2T}{n}+2\alpha QB_{\rm net}(T).
\end{equation}
\emph{Step 3: extract the three finite-time guarantees.}
Discarding the nonnegative direction sum and dividing by $\alpha T/2$
gives\,\eqref{cv:finite-bound}. Discarding the gradient sum instead and
dividing by $\alpha T/8$ gives\,\eqref{cv:finite-direction}. Inserting
the latter into\,\eqref{cv:closed-consensus} gives
\eqref{cv:finite-consensus}.
\end{proof}

\section{Parameter choices and rates}\label{app:schedules}

The finite-time bound becomes informative once a parameter schedule
satisfies every feedback condition uniformly over the horizon.
We first verify such a schedule and then expand its stationarity
and consensus terms.
This intermediate specialization sets $\nu_{\rm dp}=0$ to isolate the
rate calculation.
The private extension is derived at the end of this section.

\begin{theorem}[Stationarity and consensus rates]
\label{cv:scheduled-theorem}
Under Assumptions~\ref{ass:objectives}--\ref{ass:clipping}, fix $b,n$.
Set $\nu_{\rm dp}=0$ and choose
\begin{equation}\label{cv:scheduled-batch}
b_0=\lceil b(T+1)^{1/3}\rceil.
\end{equation}
For finite-slot Poisson sampling, restrict $T$ to values with $b_0\leq N$.
Fix $c_\gamma\in(0,1]$ and choose a sufficiently small
$c_\alpha>0$, depending only on $c_\gamma,b,n$. Set
\begin{equation}\label{cv:schedule}
\alpha=\frac{c_\alpha(1-\sqrt{\lambda})}{L_s(T+1)^{1/3}},\quad
\gamma=\frac{c_\gamma}{(T+1)^{2/3}}.
\end{equation}
These choices satisfy\,\eqref{cv:stability}. Collect the variance,
optimization, and initial-heterogeneity terms as
\begin{equation}\label{cv:rate-expression}
R_T:=
\frac{L_s\Delta_f^0}{(1-\sqrt{\lambda}) T^{2/3}}
+\frac{\Sigma^2}{nT^{2/3}}
+\frac{\lambda\Sigma^2}{T^{4/3}}
+\frac{\lambda\Sigma^2}{bnT^{2/3}}
+\frac{\lambda\zeta_0^2}{1+\sqrt\lambda}
 \left[\frac1{T^{5/3}}+\frac{1}{bnT}\right].
\end{equation}
For sufficiently large $T$, the rates are
\begin{align}
\frac1T\sum_{t=0}^{T-1}\mathbb E\|\nabla f(\bar x^t)\|^2
&=\mathcal O(R_T),
\label{cv:rate-stationarity}\\
\frac1{nT}\sum_{t=0}^{T-1}C_t
&=\mathcal O\Bigg(
\frac{\lambda\Sigma^2}{L_s^2T^{4/3}}
+\frac{\lambda\zeta_0^2}{L_s^2(1+\sqrt\lambda)T^{5/3}}
\nonumber\\
&\quad+\left[
\frac{\lambda(1-\sqrt{\lambda})^2}{bL_s^2T^{2/3}}
+\frac{\lambda(1-\sqrt{\lambda})}{L_s^2(1+\sqrt\lambda)T^{4/3}}
\right]R_T\Bigg).
\label{cv:rate-consensus}
\end{align}
The hidden constants may depend on $c_\alpha,c_\gamma,b,n$
and are independent of~$T$, $\lambda$, the initial function gap,
$\zeta_0^2$, and both variance levels.
The displayed factors therefore retain the dependence on topology,
initialization, and stochastic error.
\end{theorem}

\begin{proof}
\emph{Step 1: verify every feasibility inequality.}
Direct substitution in
\eqref{cv:network-constants} and~\eqref{cv:closure-coefficients} gives
\begin{align}
K_H&=\frac{16c_\alpha^2\lambda}{L_s^2(T+1)^{2/3}},\quad
K_U=\frac{8nc_\alpha^4\lambda(1-\sqrt{\lambda})}{L_s^2(1+\sqrt{\lambda})(T+1)^{4/3}},
\label{cv:scheduled-K}\\
A_C&=\frac{128\lambda c_\alpha^2}{bc_\gamma},\quad
\frac{2L_s^2\alpha^2}{bn\gamma}
=\frac{2c_\alpha^2(1-\sqrt{\lambda})^2}{bnc_\gamma},
\label{cv:scheduled-small-gains}\\
QA_U
&=\frac{32\lambda c_\alpha^4(1-\sqrt{\lambda})^2}{bc_\gamma (T+1)^{2/3}}
+\frac{8\lambda c_\alpha^4(1-\sqrt{\lambda})}{(1+\sqrt{\lambda})(T+1)^{4/3}}
\nonumber\\ &\quad+\frac{256\lambda c_\alpha^4(1-\sqrt{\lambda})^2}{b^2nc_\gamma^2}
+\frac{64\lambda c_\alpha^4(1-\sqrt{\lambda})}{bnc_\gamma(1+\sqrt{\lambda})(T+1)^{2/3}}.
\label{cv:scheduled-feedback}
\end{align}
Bounding the topology and horizon factors by one, it suffices
to choose $c_\alpha\leq1/4$ so that
\begin{gather*}
\frac{128c_\alpha^2}{bc_\gamma}\leq\frac12,\\
\frac{2}{bnc_\gamma}c_\alpha^2
+2\left[\frac{32}{bc_\gamma}+8
+\frac{256}{b^2nc_\gamma^2}
+\frac{64}{bnc_\gamma}\right]c_\alpha^4
\leq\frac18.
\end{gather*}
Such a positive choice exists by continuity at zero and is independent
of $T$ and $\lambda$. In addition,
$\alpha\leq(1-\sqrt{\lambda})/(4L_s)\leq1/(4L_s)<1/(2L_s)$.
Thus the prescribed schedule satisfies~\eqref{cv:stability}.

\emph{Step 2: substitute into the stationarity bound term by term.}
The same substitution, with $h_0,H_0$ still retained, yields
\begin{align}
\frac1T\sum_t\mathbb E\|\nabla f(\bar x^t)\|^2
&\leq\frac{2L_s\Delta_f^0(T+1)^{1/3}}{c_\alpha(1-\sqrt{\lambda}) T}
+\frac{2h_0(T+1)^{2/3}}{c_\gamma T}
+\frac{4c_\gamma\Sigma^2}{n(T+1)^{2/3}}\nonumber\\
&\quad+4\left(\frac{L_s^2}{n}
+\frac{8L_s^2(T+1)^{2/3}}{bn^2c_\gamma}\right)
\nonumber\\
&\quad\times\left(\frac{4nc_\alpha^2\lambda\zeta_0^2}
{L_s^2(1+\sqrt\lambda)T(T+1)^{2/3}}
+\frac{16c_\alpha^2\lambda H_0}{L_s^2c_\gamma T}
+\frac{32c_\alpha^2c_\gamma n\lambda\Sigma^2}{L_s^2(T+1)^{4/3}}\right).
\label{cv:scheduled-exact-bound}
\end{align}
Here the second parenthesis is exactly $B_{\rm net}(T)/T$, because
\begin{equation}\label{cv:scheduled-network-budget}
\frac{B_{\rm net}(T)}T
=\frac{4nc_\alpha^2\lambda\zeta_0^2}
{L_s^2(1+\sqrt\lambda)T(T+1)^{2/3}}
+\frac{16c_\alpha^2\lambda H_0}{L_s^2c_\gamma T}
+\frac{32c_\alpha^2c_\gamma n\lambda\Sigma^2}{L_s^2(T+1)^{4/3}}.
\end{equation}
Under $\nu_{\rm dp}=0$,\,\eqref{cv:initial-errors} and
\eqref{cv:scheduled-batch} give
\[
h_0=\mathcal O\!\left(\frac{\Sigma^2}{n(T+1)^{1/3}}\right),\quad
H_0=\mathcal O\!\left(\frac{n\Sigma^2}{(T+1)^{1/3}}\right),
\]
using $\sigma_{\rm sgd}^2/b\leq\Sigma^2$.
Using $T\leq T+1\leq2T$, the first three terms of
\eqref{cv:scheduled-exact-bound} are therefore bounded, respectively,
by constant multiples of
\[
\frac{L_s\Delta_f^0}{(1-\sqrt{\lambda}) T^{2/3}},\quad
\frac{\Sigma^2}{nT^{2/3}},\quad
\frac{\Sigma^2}{nT^{2/3}}.
\]
The two variance terms of\,\eqref{cv:scheduled-network-budget} are each bounded by
a constant multiple of $n\lambda\Sigma^2/(L_s^2T^{4/3})$.
Multiplying that estimate by the two terms of
$4Q=4L_s^2/n+32L_s^2(T+1)^{2/3}/(bn^2c_\gamma)$ yields, respectively,
\[
\mathcal O\left(\frac{\lambda\Sigma^2}{T^{4/3}}\right),\quad
\mathcal O\left(\frac{\lambda\Sigma^2}{bnT^{2/3}}\right).
\]
The remaining initial-heterogeneity term in
\eqref{cv:scheduled-network-budget}, multiplied by $4Q$, is exactly
\[
\underbrace{\frac{16c_\alpha^2\lambda\zeta_0^2}
{(1+\sqrt\lambda)T(T+1)^{2/3}}}_{\text{initial energy in average descent}}
+\underbrace{\frac{128c_\alpha^2\lambda\zeta_0^2}
{bn c_\gamma(1+\sqrt\lambda)T}}_{\text{initial energy through recursive differences}}.
\]
This gives the last term of~\eqref{cv:rate-expression} and accounts
for every contribution to~\eqref{cv:rate-stationarity}.
For fixed initial heterogeneity, these contributions decay as
$O(T^{-1})$ or faster.
Allowing heterogeneous initial gradients therefore preserves the
leading $T^{-2/3}$ order under the stated schedule.

\emph{Step 3: substitute into the consensus bound.}
For consensus, direct substitution gives
\begin{equation}\label{cv:scheduled-AU}
\frac{A_U}{n}
=\frac{32\lambda c_\alpha^4(1-\sqrt{\lambda})^2}{bc_\gamma L_s^2(T+1)^{2/3}}
+\frac{8\lambda c_\alpha^4(1-\sqrt{\lambda})}{L_s^2(1+\sqrt{\lambda})(T+1)^{4/3}}.
\end{equation}
The right-hand side of\,\eqref{cv:finite-direction} is exactly four
times that of\,\eqref{cv:finite-bound}. The term-by-term estimates just
proved hence also give $T^{-1}\sum_{t=0}^{T-1}U_t=\mathcal O(R_T)$.
Divide\,\eqref{cv:closed-consensus} by $nT$; its first term is
$\mathcal O(\lambda\Sigma^2/(L_s^2T^{4/3})
+\lambda\zeta_0^2/[L_s^2(1+\sqrt\lambda)T^{5/3}])$ by
\eqref{cv:scheduled-network-budget}, while its second is
$2(A_U/n)\,\mathcal O(R_T)$. Equation\,\eqref{cv:scheduled-AU}
now gives\,\eqref{cv:rate-consensus} with all factors of $(1-\sqrt{\lambda})$ and
$1+\sqrt{\lambda}$ retained. The derivation uses only constants with the dependence
stated in the theorem. At $\lambda=0$, every displayed consensus term is
zero, consistently with Lemma~\ref{cv:lemma-network}.
\end{proof}

With $\nu_{\rm dp}=0$, uniform independent selection from
$\bar x^0,\ldots,\bar x^{T-1}$ gives expected squared gradient norm equal
to the average in~\eqref{cv:rate-stationarity}. Substituting
$\Sigma^2=\sigma_{\rm sgd}^2/b$ gives expected squared gradient norm
$\mathcal O(T^{-2/3})$ for fixed problem, network and batch parameters.
Cauchy--Schwarz gives expected gradient norm
$\mathcal O(T^{-1/3})$.
The per-node gradient count is $b_0+2b(T-1)$, interpreted in
expectation for Poisson sampling.
Since $b_0=\mathcal O(bT^{1/3})$, the per-node oracle complexity
is $\mathcal O(bT)$ over $T$ iterations.

\paragraph{Common-variance private extension.}
Retain the DP variance in~\eqref{cv:initial-errors} and
\eqref{cv:unified-variance}.
The sampling parts yield the scheduled terms already proved above.
The DP noise variance contributes to~\eqref{cv:finite-bound}
\[
d\nu_{\rm dp}^2\left[
\frac{2}{n\gamma T}+\frac{4\gamma}{n}
+4QK_Hn\left(\frac1{\gamma T}+2\gamma\right)\right].
\]
Using~\eqref{cv:scheduled-K}, the first two terms are respectively
$\mathcal O(d\nu_{\rm dp}^2/(nT^{1/3}))$ and
$\mathcal O(d\nu_{\rm dp}^2/(nT^{2/3}))$.
The two components of $Q$ applied to $4K_Hn/(\gamma T)$
give $\mathcal O(d\nu_{\rm dp}^2\lambda/T)$ and
$\mathcal O(d\nu_{\rm dp}^2\lambda/(bnT^{1/3}))$.
Those applied to $8K_Hn\gamma$ give
$\mathcal O(d\nu_{\rm dp}^2\lambda/T^{4/3})$ and
$\mathcal O(d\nu_{\rm dp}^2\lambda/(bnT^{2/3}))$.
The latter two terms and the averaged iterative term
are bounded by their corresponding initial terms.
Thus the complete DP contribution is
\[
\mathcal O\!\left(d\nu_{\rm dp}^2
 \left[\frac1{nT^{1/3}}+\frac{\lambda}{T}
 +\frac{\lambda}{bnT^{1/3}}\right]\right).
\]
This is a bound for fixed $\nu_{\rm dp}^2$; under a fixed privacy
budget its horizon dependence must also be substituted.

The terms displayed in the main theorem follow from
\[
T^{-5/3}+(bnT)^{-1}\leq(1+bn)(bnT)^{-1},
\]
and
\[
1+\lambda nT^{-2/3}+\lambda/b
\leq(1+n)(1+\lambda/b).
\]
These estimates preserve the stated independence from $T$ and $\lambda$
of the hidden constants.

\section{Privacy proof}\label{app:privacy}
Each local data set is represented by $N$ public slots.
An empty slot $\bot$ has zero loss and gradient.
Add/remove adjacency replaces one real record by $\bot$, or conversely,
at one node. The capacity $N$ and the sampling probabilities $b_0/N$ and $b/N$
are unchanged on adjacent data sets.
For each fresh batch,
\begin{equation}\label{pv:poisson-estimator}
\widehat g_i(x;\mathcal B_i^t)=\frac1b
\sum_{r\in\mathcal B_i^t}\nabla\ell_i(x;\zeta_{ir}).
\end{equation}
The protected node keeps its sampled identities and Gaussian draws private.

\paragraph{R\'enyi accounting.}
We use R\'enyi differential privacy (RDP) to compose adaptive releases
and account for amplification by subsampling~\cite{mironov2017rdp}.
For probability laws $\mu,\mu'$ and order $a>1$, define
\begin{equation}\label{eq:renyi-divergence}
D_a(\mu\|\mu'):=\frac1{a-1}
\log\int\left(\frac{d\mu}{d\mu'}\right)^a d\mu',
\end{equation}
with divergence $+\infty$ when $\mu$ is not absolutely continuous with
respect to $\mu'$.

\begin{definition}[Record-level R\'enyi differential privacy]\label{def:record-rdp}
A randomized mechanism $\mathcal A$ is $(a,\varepsilon_a)$-RDP if,
for every ordered adjacent pair $D\sim D'$, the laws $\mu_D,\mu_{D'}$
of its outputs satisfy
\begin{equation}\label{eq:rdp-definition}
D_a(\mu_D\|\mu_{D'})\leq\varepsilon_a.
\end{equation}
\end{definition}
At the same order $a$, adaptive composition adds the conditional
RDP costs of the fresh releases.
An $(a,\varepsilon_a)$-RDP mechanism is also
$(\varepsilon,\delta)$-DP whenever
\begin{equation}\label{eq:rdp-conversion}
\varepsilon_a+\frac{\log(1/\delta)}{a-1}\leq\varepsilon.
\end{equation}
This conversion~\cite{mironov2017rdp} gives the transcript privacy
guarantee used in the analysis.

\paragraph{Sensitivity and Gaussian perturbation.}
For the fresh query, Euclidean sensitivity at a fixed history and mask is
\begin{equation}\label{eq:query-sensitivity}
\Delta_{q_i^t}=\sup_{D\sim D'}
\|q_i^t(D;\mathcal B)-q_i^t(D';\mathcal B)\|.
\end{equation}
The recursive perturbation has covariance $\gamma^2\nu_{\rm dp}^2I_d$,
whereas initialization uses $\nu_0^2I_d$.
Privacy depends on the noise standard deviation relative to sensitivity.

\paragraph{Reconstruction from the transcript.}
The messages $\phi_i^t$ determine $x_i^t=\sum_jW_{ij}\phi_j^t$.
Starting from the public $\psi_i^0=x_i^0$, the identities
$\psi_i^t=\phi_i^t-x_i^{t-1}+\psi_i^{t-1}$ and
$v_i^{t-1}=(x_i^{t-1}-\psi_i^t)/\alpha$ recover historical directions.
Thus the current and previous model points and the old direction
are fixed by the past transcript.
The fresh query and noisy direction are
\begin{equation}\label{pv:conditional-query}
q_i^t=\gamma g_i^t+(1-\gamma)d_i^t,
\quad v_i^t=(1-\gamma)v_i^{t-1}+q_i^t+\gamma\xi_i^t.
\end{equation}
Let $D\sim D'$ differ at slot $r_*$ of node $i$, with record
$\zeta$ in $D$ and $\bot$ in $D'$. Fix the same past transcript
and the same Poisson sampling mask $\mathcal B$ on both data sets.
The model points are then identical, and every unchanged slot cancels:
\[
\begin{aligned}
&q_i^t(D;\mathcal B)-q_i^t(D';\mathcal B)\\
&=\frac{\mathbf1\{r_*\in\mathcal B\}}b
\Big\{\gamma\operatorname{clip}_{C_g}(\nabla\ell_i(x_i^t;\zeta))
+(1-\gamma)\operatorname{clip}_{C_\Delta}
(\nabla\ell_i(x_i^t;\zeta)-\nabla\ell_i(x_i^{t-1};\zeta))\Big\}.
\end{aligned}
\]
Reversing the adjacent pair changes only the sign. Applying the
clipping bounds and taking the supremum over adjacent pairs yields
\[
\begin{aligned}
\Delta_{q_i^t}
&\leq\frac{\gamma C_g+(1-\gamma)C_\Delta}{b}
=\frac{S_\gamma}{b}.
\end{aligned}
\]
If the changed slot is not sampled, the difference is zero.
The sampling probability $b/N$ enters the subsampled Gaussian
accountant below, rather than multiplying this worst-case sensitivity.
The same argument gives an initial gradient-query sensitivity bound
$C_g/b_0$.

\paragraph{Exact finite-order accountant.}\label{app:accountant}
For integer $a\geq2$, the Poisson-sampled Gaussian RDP bound is
\begin{equation}\label{pv:rdp-function}
\varepsilon_{\rm SG}(a;q,z)=\frac1{a-1}
\log\left[\sum_{k=0}^a\binom ak(1-q)^{a-k}q^k
 \exp\!\left(\frac{k(k-1)}{2z^2}\right)\right].
\end{equation}
Here $z$ is the Gaussian noise standard deviation divided by the
query sensitivity bound, and $q$ is the Poisson sampling probability.
The accountant combines these two quantities to bound the per-round
RDP cost after amplification by subsampling: at fixed $z$, reducing
$q$ lowers this bound. The amplification comes from private random
sampling, not from $z$ itself; $q=1$ gives the unsampled Gaussian
cost $a/(2z^2)$.
The sampled-Gaussian reduction bounds both neighboring divergence
orders by this expression~\cite[Cor.~7 and Sec.~3.3]{mironov2019sgm}.
Expanding the likelihood ratio
$(1-q)+q\exp((2u-1)/(2z^2))$ to integer power $a$ and integrating
against $u\sim\mathcal N(0,z^2)$ gives the displayed finite sum.

\begin{lemma}[Transcript privacy with DP noise]\label{pv:privacy-utility}
Let $b_0\geq b$, and choose $z>0$ and an integer $a\geq2$ satisfying
\begin{equation}\label{pv:accountant-test}
\varepsilon_{\rm SG}(a;b_0/N,z)
+(T-1)\varepsilon_{\rm SG}(a;b/N,z)
+\frac{\log(1/\delta)}{a-1}\leq\varepsilon.
\end{equation}
The DP noise variances
\begin{equation}\label{pv:main-noise-orders}
\nu_0^2=\frac{z^2C_g^2}{b_0^2},\qquad
\nu_{\rm dp}^2=\frac{z^2S_\gamma^2}{\gamma^2b^2}
\end{equation}
give $(\varepsilon,\delta)$-DP for the complete transcript and
for an honest-but-curious node's view of records held elsewhere.
\end{lemma}
\begin{proof}
\emph{Step 1: account for the recursive and initial queries.}
For each recursive release, its noise standard deviation
$\gamma\nu_{\rm dp}$ divided by the sensitivity bound $S_\gamma/b$
equals $z$.
For the initial release, $\nu_0$ divided by $C_g/b_0$ also equals $z$.
Its cost is therefore $\varepsilon_{\rm SG}(a;b_0/N,z)$; each later release costs
at most $\varepsilon_{\rm SG}(a;b/N,z)$.

\emph{Step 2: compose all communication rounds.}
The communication record consists of all messages
$\{\phi_i^{t+1}:i=1,\ldots,n,\ t=0,\ldots,T-1\}$.
Fix the changed node and a common past transcript.
All other nodes have the same conditional kernels on the neighboring
data sets, so the joint release incurs one changed node's cost,
not $n$ times that cost.
Adaptive composition and~\eqref{eq:rdp-conversion} establish the budget.

\emph{Step 3: include the internal observer's information.}
For an internal observer other than the changed node, its local
data are fixed on the adjacent pair.
Condition additionally on its data-independent random tape.
Its messages are then deterministic functions of its data, tape,
and the past transcript, while the protected node retains its fresh
sampling and Gaussian draws.
The same conditional bounds hold uniformly in the tape.
Integrating conditional R\'enyi moments against the common tape law
preserves the RDP bound.
The received-message view is a projection and inherits DP by
post-processing~\cite{dwork2014algorithmic}.

Adaptation, correction, and mixing are deterministic post-processing
of the noisy directions.
The reconstruction identities show that mixing itself supplies no
privacy amplification against the full-transcript observer.
\end{proof}

\paragraph{Explicit sufficient variance.}
Choose $a=\lceil2\log(1/\delta)/\varepsilon\rceil+1$ and set
\begin{equation}\label{pv:explicit-multipliers}
z^2=\max\left\{\frac{a(a-1)}2,\,
 \frac{2(e-1)a}{\varepsilon}
 \left[\left(\frac{b_0}{N}\right)^2
 +(T-1)\left(\frac bN\right)^2\right]\right\}.
\end{equation}
For $K\sim\operatorname{Binomial}(a,q)$, the first condition makes
$u=K(K-1)/(2z^2)\in[0,1]$.
Using $e^u\leq1+(e-1)u$ and $\log(1+x)\leq x$ yields
\[
\varepsilon_{\rm SG}(a;q,z)
\leq\frac{(e-1)a q^2}{2z^2}.
\]
Applying the same bound to the initialization rate $b_0/N$ and the
$T-1$ recursive rates $b/N$, Equation~\eqref{pv:explicit-multipliers}
bounds the composition cost
by $\varepsilon/4$ and the conversion by $\varepsilon/2$,
and therefore satisfies~\eqref{pv:accountant-test}.
For $0<\varepsilon\leq1$, $0<\delta<1/2$,
$a=\mathcal O(\log(1/\delta)/\varepsilon)$, so substitution into
\eqref{pv:main-noise-orders} gives~\eqref{eq:dp-noise-variance}.
Both terms in the maximum are retained, avoiding an unjustified
quadratic sampling approximation at arbitrary noise levels.
The exact accountant can give substantially smaller noise.
If $C_g=C_\Delta=0$, no privacy noise is required.

\paragraph{Proof of Theorem~\ref{thm:utility}.}\label{app:utility}
\begin{proof}
The stationarity bound~\eqref{cv:rate-stationarity}, with
$\Sigma^2=\sigma_{\rm sgd}^2/b$, gives the optimization,
initial-heterogeneity, and sampling terms in~\eqref{eq:private-utility}.
The private extension in Appendix~\ref{app:schedules} adds the
DP contribution
\[
\mathcal O\!\left(
\frac{d\nu_{\rm dp}^2}{nT^{1/3}}
\left(1+\frac{\lambda}{b}\right)\right).
\]
Substituting the noise variance~\eqref{eq:dp-noise-variance}
gives its privacy term
\[
\mathcal O\!\left(
\frac{dS_\gamma^2}{n\gamma^2\varepsilon^2T^{1/3}}
\left(1+\frac{\lambda}{b}\right)
\left[\frac{T\log(1/\delta)}{N^2}
+\frac{\log^2(1/\delta)}{b^2}\right]\right).
\]
Combining these four contributions establishes~\eqref{eq:private-utility}.
Lemma~\ref{pv:privacy-utility} supplies the stated privacy guarantee,
including initialization. This completes the proof of
Theorem~\ref{thm:utility}.
\end{proof}

\section{Sensitivity example}
\label{app:sensitivity-example}
Consider $d=1$ and the finite-sum losses
\begin{equation}\label{eq:logcosh-example}
\ell_i(x;\zeta_{ir})=a_{ir}x+\alpha\log\cosh x,
\qquad |a_{ir}|\leq A,
\end{equation}
where the coefficients satisfy
\[
\frac1{nN}\sum_{i=1}^n\sum_{r=1}^N a_{ir}=0,
\qquad
\bar a_i:=\frac1N\sum_{r=1}^N a_{ir}
\quad\text{are not all equal}.
\]
Run Algorithm~\ref{alg:storm-ed} with
$0<\alpha<\min\{1,A\}$ and Poisson batches of expected size $b<N$.
Let $n\geq4$ and take a ring network with
$W_{ii}=1/2$ and $W_{i,i-1}=W_{i,i+1}=1/4$, with indices modulo $n$.

\begin{proposition}[Small recursive sensitivity]
\label{prop:sparse-sensitivity}
For the objective~\eqref{eq:logcosh-example}, Assumptions
\ref{ass:objectives}--\ref{ass:clipping} hold with
\[
L_s=1,\qquad \sigma_{\rm sgd}=A+1,\qquad
C_g=A+\alpha,\qquad C_\Delta=2\alpha.
\]
For every past transcript, sampling mask, and realization of the
Gaussian noise, the conditional query sensitivity satisfies
\begin{equation}\label{eq:example-sensitivity}
\Delta_{q_i^t}\leq
\frac{\gamma A+(2-\gamma)\alpha}{b}
=\mathcal O\!\left(\frac{\gamma+\alpha}{b}\right).
\end{equation}
This is strictly smaller than $(A+\alpha)/b$, the corresponding
$\gamma=1$ bound, whenever $\gamma<1$.
\end{proposition}

\begin{proof}
The record gradient and Hessian are
\[
\nabla\ell_i(x;\zeta_{ir})=a_{ir}+\alpha\tanh x,
\qquad
\nabla^2\ell_i(x;\zeta_{ir})=\alpha\operatorname{sech}^2x.
\]
Since $\alpha<1$, the Hessian is bounded by one, so
Assumption~\ref{ass:objectives} holds with $L_s=1$.
The chosen ring matrix is symmetric, nonnegative, doubly stochastic,
and supported on a connected graph.  Its eigenvalues are
$[1+\cos(2\pi k/n)]/2\in[0,1]$, so
Assumption~\ref{ass:network} holds.

For a point fixed before the current Poisson batch is drawn,
normalization by the expected batch size gives
\[
\E_t[\widehat g_i(x;\mathcal B_i^t)]
=\bar a_i+\alpha\tanh x=\nabla f_i(x).
\]
Moreover,
\[
\E_t|\widehat g_i(x;\mathcal B_i^t)-\nabla f_i(x)|^2
=\frac{1-b/N}{bN}\sum_{r=1}^N
|a_{ir}+\alpha\tanh x|^2
\leq\frac{(A+1)^2}{b}.
\]
Thus Assumption~\ref{ass:oracle} holds with
$\sigma_{\rm sgd}=A+1$.  The local objectives
$f_i(x)=\bar a_i x+\alpha\log\cosh x$ are heterogeneous, while
$f(x)=\alpha\log\cosh x$ is bounded below.

For all records and all $x,y\in\mathbb R$,
\begin{align*}
|\nabla\ell_i(x;\zeta_{ir})|&\leq A+\alpha=C_g,\\
|\nabla\ell_i(x;\zeta_{ir})-\nabla\ell_i(y;\zeta_{ir})|
&=\alpha|\tanh x-\tanh y|\leq2\alpha=C_\Delta.
\end{align*}
Hence both clipping operations are inactive for every stochastic and
noisy trajectory, and Assumption~\ref{ass:clipping} holds.

Fix any past transcript and sampling mask in the conditional
sensitivity definition~\eqref{eq:query-sensitivity}.  The current and
previous models are then fixed, and a selected added record contributes
at most
\[
\frac{\gamma C_g+(1-\gamma)C_\Delta}{b}
=\frac{\gamma A+(2-\gamma)\alpha}{b}.
\]
This proves~\eqref{eq:example-sensitivity}.  Finally,
\[
C_g-S_\gamma=(1-\gamma)(A-\alpha)>0
\]
for $\gamma<1$, which proves the strict comparison.
\end{proof}

\section{Experimental settings}\label{app:experiments}
{The following settings specify the experiments in
Fig.~\ref{fig:rq1}, Table~\ref{tab:cifar-private}, and
Fig.~\ref{fig:cifar-private}.}
We report initialization and per-round work alongside training
parameters because equal communication rounds need not imply equal
gradient-evaluation budgets.
The logistic-regression experiment used no privacy noise, whereas
the CIFAR-10 comparison used matched per-run privacy targets.

\subsection{Nonprivate heterogeneous logistic regression}
There are $n=32$ nodes, each holding $N=2000$ records in dimension $d=20$.
For the $r$th sample at node $i$, $a_{ir}\in\mathbb R^{20}$ is its
feature vector and $y_{ir}\in\{-1,+1\}$ is its binary class label.
The local objective for the model parameter $x\in\mathbb R^{20}$ is
\begin{equation}
f_i(x)=\frac1N\sum_{r=1}^N
\log\!\left(1+\exp(-y_{ir}a_{ir}^{\top}x)\right)
+0.001\sum_{j=1}^{20}\frac{x_j^2}{1+x_j^2}.
\end{equation}
We drew independent features $a_{ir}\sim\mathcal N(0,I_{20})$,
a common teacher $w\sim\mathcal N(0,I_{20})$, and independent
node perturbations $h_i\sim\mathcal N(0,0.2I_{20})$.
Conditional on these variables, each label was sampled independently.
The probability of $y_{ir}=+1$ was
$\operatorname{sigmoid}(a_{ir}^{\top}(w+h_i))$; otherwise,
$y_{ir}=-1$.
The generated features and labels remained fixed throughout training.
Each node shared one perturbation $h_i$ across its local records.
Nodes therefore had the same feature distribution but different
conditional label distributions, with heterogeneity set by covariance
$0.2I_{20}$.
We used the generated features without clipping or rescaling.
The $32$ nodes formed an undirected ring, each communicating
with its two neighbors.
The first and last nodes were connected, completing the ring.
Mixing assigned weight $1/2$ to the node itself and $1/4$ to
each neighbor; all other weights were zero.
All methods started from $x_i^0=0$ and used constant stepsize
$\alpha=1$ for $600$ communication rounds.
PRDO used momentum parameter $\gamma=0.1$ throughout both
the full-batch and minibatch comparisons.
EDM and Momentum Tracking used normalized momentum with history
weight $\beta=0.9$ and current-gradient weight $1-\beta=0.1$.
Their auxiliary states were initialized to zero before the first
gradient update.

The full-batch experiment evaluated all $2000$ local records
at each gradient computation.
In the minibatch experiment, each node independently sampled $16$
records uniformly without replacement at every round.
PRDO used the full local gradient for initialization, then
evaluated each fresh batch at the current and previous iterates.
The baselines used batch size $16$ starting from their first update.
Five paired sampling streams were shared across methods for
the minibatch comparison.
This experiment used neither privacy noise nor gradient clipping
in any of the methods.

\subsection{Private CIFAR-10}
{
We evaluated our PRDO,
DP Exact Diffusion (DP-ED), DP
Exact Diffusion with Momentum (DP-EDM), DP Momentum
Tracking (DP-MT), and DP decentralized stochastic
gradient descent (DP-DSGD) on CIFAR-10 using cross-entropy loss.
The data comprised $45000$ training,
$5000$ validation, and $10000$ test images. We distributed the training
set across ten nodes using a balanced Dirichlet partition with
concentration $0.1$, assigning $4500$ records to each node.

All methods used the same VGG-style network with six convolutional
layers, three max-pooling layers, and two fully connected layers.
The hidden layer had width $128$ and tanh activation. The nodes formed
a lazy ring with self-weight $1/2$ and weight $1/4$ on each neighbor.

Training ran for $T=9500$ communication rounds with constant stepsize
$\alpha=0.05$ and expected batch sizes $b_0=b=100$. Each node used
Poisson sampling with rate $q=100/4500$. Per-record gradients were
clipped at $C_g=1$. PRDO used $\gamma=0.05$ and clipped each
same-record gradient difference at $C_\Delta=0.001$. DP-EDM and DP-MT
used history weight $\beta=0.95$. DP-DSGD used neither a recursive
gradient difference nor momentum.

We considered add/remove record-level privacy with
$\varepsilon\in\{4,6,8\}$ and $\delta=10^{-5}$. For each privacy
budget, the Opacus PRV accountant calibrated the Gaussian noise
multiplier over $9500$ queries. The resulting multipliers were
$2.464143$, $1.807607$, and $1.480246$ for
$\varepsilon=4,6,8$, respectively. PRDO applied the corresponding
noise according to Algorithm~\ref{alg:storm-ed}.
Each entry in Table~\ref{tab:cifar-private} is the arithmetic mean of
the test accuracies at rounds $8500$, $8750$, $9000$, $9250$, and
$9500$.
}

\begin{figure}[!h]
\centering
\includegraphics[width=.95\textwidth]{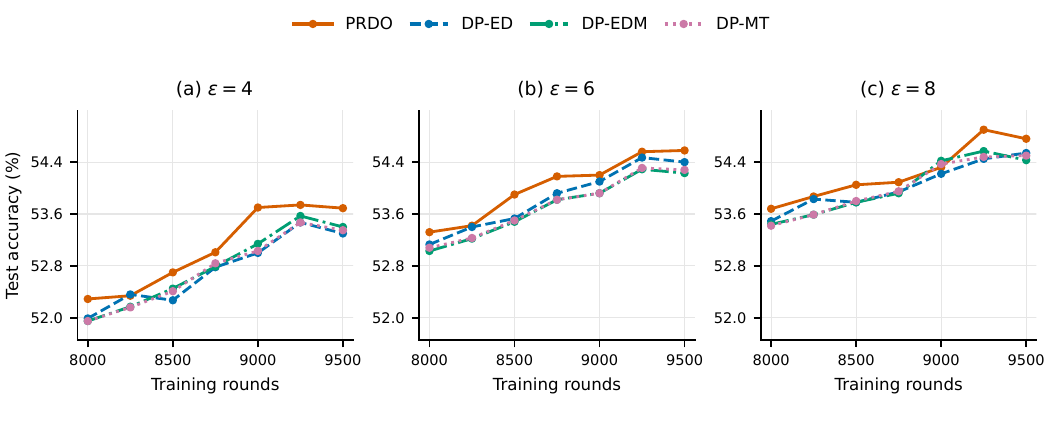}
\caption{CIFAR-10 test accuracy under $\varepsilon=4,6,8$ and $\delta=10^{-5}$. DP-DSGD is
reported in Table~\ref{tab:cifar-private} and omitted here for
readability.}
\label{fig:cifar-private}
\end{figure}

\end{document}